\documentclass{article}

\title{\huge {Ideal Abstractions for Decision-Focused Learning}}
 \author{Michael Poli\footnote{Equal contribution. $^1$Microsoft Research. $^2$Stanford University. $^3$Mila and Universit\'e de Montr\'eal. $^4$ Carnegie Mellon University.}~$^{,1,2}$, Stefano Massaroli$^{*,3}$ \\  Stefano Ermon$^{2}$, Bryan Wilder$^4$, Eric Horvitz$^1$
 }
 
\date{\small{\footnotesize\sf Version}: AISTATS 2023 camera-ready, {\footnotesize\sf  Last Compiled}: \today}
\usepackage{lipsum}
\usepackage[utf8]{inputenc}
\usepackage[T1]{fontenc}    
\usepackage[bookmarks=true, colorlinks=true,linkcolor=blue!70,citecolor=green!80!black]{hyperref}
\usepackage[round]{natbib}
\usepackage{wrapfig}
\usepackage{booktabs}  
\usepackage{longtable}
\usepackage{ducksay}
\usepackage{multicol}
\usepackage{etaremune}
\usepackage{afterpage}
\usepackage{capt-of}
\usepackage[table, x11names]{xcolor}
\usepackage{decorule}
\usepackage{scalerel,xparse}
\usepackage{enumitem} 
\usepackage[top=25mm, bottom=25mm, left=25mm, right=25mm,
            foot=5mm, marginparsep=0mm]{geometry}

\usepackage{pythonhighlight}
\usepackage{chngcntr}
\counterwithin{figure}{section} 
\counterwithin{table}{section}

\usepackage{graphicx}
\usepackage{tikz}
\usepackage{tikz-cd}
\usepackage{hf-tikz}
\usepackage{pgfplots} 
\pgfplotsset{compat=newest} 
\pgfplotsset{
        table/search path={figures/drawings},
    }
\usetikzlibrary{fadings}
\usetikzlibrary{shapes, arrows, fit, backgrounds, arrows.meta}

\usetikzlibrary{matrix}
\usetikzlibrary{shadows.blur}
\usetikzlibrary{patterns, tikzmark}
\usetikzlibrary{decorations.pathreplacing, calc, decorations.markings,}
\usetikzlibrary{positioning}
\usetikzlibrary{shapes.geometric}
\pgfdeclareplotmark{mystar}{
    \node[star,star point ratio=2.25,minimum size=5pt,
          inner sep=0pt,draw=black,solid,fill=green] {};
}

\usepgfplotslibrary{groupplots}
\usepgfplotslibrary{patchplots}

\definecolor{bg}{gray}{0.97}
\definecolor{olive}{rgb}{0.6, 0.6, 0.2}
\definecolor{sand}{rgb}{0.8666666666666667, 0.8, 0.4666666666666667}
\definecolor{wine}{rgb}{0.5333333333333333, 0.13333333333333333, 0.3333333333333333}
\definecolor{deblue}{RGB}{11,132,147}
\definecolor{ocra}{RGB}{204, 119, 34}

\usepackage{lettrine}
\usepackage{Typocaps}

\LettrineTextFont{\itshape}
\usepackage{amsthm}

\newtheorem{definition}{Definition}[section]

\usepackage{amssymb}
\usepackage{pifont}
\usepackage{minitoc}

\newcommand{\chapref}[1]{\hyperref[#1]{Chapter \ref{#1}}}
\newcommand{\secref}[1]{\hyperref[#1]{Section \ref{#1}}}

\usepackage{marginnote}

\usepackage[many]{tcolorbox}
\tcbuselibrary{breakable,xparse,skins}

\usepackage{newunicodechar}
\newunicodechar{λ}{{$\mathtt\lambda$}}
\newunicodechar{μ}{{$\mathtt\mu$}}

\usepackage{xparse}
\DeclareTColorBox{emphbox}{O{black}O{0cm}}{
    enhanced jigsaw,
    breakable,
    outer arc=0pt,
    arc=0pt,
    colback=white,
    rightrule=0pt,
    toprule=0pt,
    top=0pt,
    right=0pt,
    bottom=0pt,
    bottomrule=0pt,
    colframe=#1,
    enlarge left by=#2,
    width=\linewidth-#2,
}

\DeclareTColorBox{emphbox}{O{black}O{0cm}}{
    empty,
    breakable=true,
    outer arc=0pt,
    arc=0pt,
    rightrule=0pt,
    leftrule=2pt,
    borderline west={2pt}{0pt}{#1},
    toprule=0pt,
    top=0pt,
    right=-3pt,
    bottom=0pt,
    bottomrule=0pt,
    colframe=#1,
    enlarge left by=#2,
    width=\linewidth-#2,
}

\RequirePackage{algorithm}
\RequirePackage{algorithmic}

\usepackage[whole]{bxcjkjatype}
\usepackage{amsmath, amssymb, amsfonts, mathtools}
\usepackage{physics}
\usepackage{cancel}
\usepackage{thmtools, thm-restate, amsthm}
\usepackage{accents}

\newtheorem{prop}{Proposition}

\newcommand{\x}{\times}

\usepackage{pict2e, picture}
\makeatletter
\DeclareRobustCommand{\Arrow}[1][]{%
\check@mathfonts
\if\relax\detokenize{#1}\relax
\settowidth{\dimen@}{$\m@th\rightarrow$}%
\else
\setlength{\dimen@}{#1}%
\fi
\sbox\z@{\usefont{U}{lasy}{m}{n}\symbol{41}}%
\begin{picture}(\dimen@,\ht\z@)
\roundcap
\put(\dimexpr\dimen@-.7\wd\z@,0){\usebox\z@}
\put(0,\fontdimen22\textfont2){\line(1,0){\dimen@}}
\end{picture}%
}
\makeatother

\newcommand{\cA}{\mathcal{A}}

\newcommand{\cL}{\mathcal{L}}

\newcommand{\cP}{\mathcal{P}}

\newcommand{\cU}{\mathcal{U}}

\newcommand{\bA}{\mathbb{A}}

\newcommand{\bE}{\mathbb{E}}

\newcommand{\R}{\mathbb{R}}
\newcommand{\bS}{\mathbb{S}}

\newcommand{\bV}{\mathbb{V}}

\newcommand{\bZ}{\mathbb{Z}}

\newcommand{\eb}{\mathbf{e}}

\usepackage{bm}
\newcommand{\bp}{\bm p}
\newcommand{\bq}{\bm q}
\newcommand{\bL}{\bm L}

\DeclareMathAlphabet{\nummathbb}{U}{BOONDOX-ds}{m}{n}
\newcommand{\0}{\nummathbb{0}}

\makeatletter
\DeclareRobustCommand\widecheck[1]{{\mathpalette\@widecheck{#1}}}
\def\@widecheck#1#2{%
    \setbox\z@\hbox{\m@th$#1#2$}%
    \setbox\tw@\hbox{\m@th$#1%
       \widehat{%
          \vrule\@width\z@\@height\ht\z@
          \vrule\@height\z@\@width\wd\z@}$}%
    \dp\tw@-\ht\z@
    \@tempdima\ht\z@ \advance\@tempdima2\ht\tw@ \divide\@tempdima\thr@@
    \setbox\tw@\hbox{%
       \raise\@tempdima\hbox{\scalebox{1}[-1]{\lower\@tempdima\box
\tw@}}}%
    {\ooalign{\box\tw@ \cr \box\z@}}}
\makeatother

\begin{document}
\maketitle
 \begin{abstract}
    We present a methodology for formulating simplifying abstractions in machine learning systems by identifying and harnessing the utility structure of decisions. Machine learning tasks commonly involve high-dimensional output spaces (e.g., predictions for every pixel in an image or node in a graph), even though a coarser output would often suffice for downstream decision-making (e.g., regions of an image instead of pixels). Developers often hand-engineer abstractions of the output space, but numerous abstractions are possible and it is unclear how the choice of output space for a model impacts its usefulness in downstream decision-making. We propose a method that configures the output space automatically in order to minimize the loss of decision-relevant information. Taking a geometric perspective, we formulate a step of the algorithm as a projection of the probability simplex, termed \textit{fold}, that minimizes the total loss of decision-related information \textit{in the H-entropy sense}. Crucially, learning in the abstracted outcome space requires less data, leading to a net improvement in decision quality. We demonstrate the method in two domains: data acquisition for deep neural network training and a closed-loop wildfire management task.
\end{abstract}

\setlength\abovedisplayshortskip{2pt}
\setlength\belowdisplayshortskip{2pt}
\setlength\abovedisplayskip{2pt}
\setlength\belowdisplayskip{2pt}

\section{Introduction}
Modern machine learning systems process high-dimensional data such as gigapixel images \citep{litjens2022decade} or graphs with billions of nodes \citep{zheng2020distdgl}. How can machine learning efforts and outputs at this scale be most appropriately matched to predictions made in support of real-world decision-making? Further, how does one go about handling domains where the dimensionality of the problem is so large that one cannot simply collect enough data for a predictive model to ``explore'' its ambient space? 

It has been shown that if collecting a sufficient amount of data is possible, deep learning provides effective methods to compress the information content into a set of parameters \citep{bommasani2021opportunities} which can then be adapted to overcome data constraints in other similar tasks. We focus on domains where it is not possible to acquire enough data for systematic generalization of large models to occur, based in the intrinsic properties of the domain, e.g., sufficient data simply does not exist \citep{hersbach2020era5} or is too expensive to acquire. We introduce and develop a framework to tame this fundamental challenge with the traditional collect-data-and-compute--first approach by incorporating knowledge about downstream tasks. The key direction of distilling ideal abstractions for decision-focused machine learning is inspired by earlier work on utility, abstraction, and information selection in a decision-making setting \citep{horvitz1993abstract,poh1994,horvitz1995display,bach2006asymmetric,kapoor2009boundaries, azuma2006review}.

In this work, we adopt a decision-theoretic perspective to machine learning. We derive a computationally efficient method to abstract away information that is not relevant to the decision task at hand, harnessing clues about problem structure, and in the process, reduce the dimensionality of upstream prediction problems. Concretely, we cast the search for the right abstractions into an optimization problem based on a geometric perspective. We introduce a class of algorithms we refer to as {\tt ORIGAMI} that iteratively aggregate sets of outcomes through projections, termed \textit{folds}, of the probability simplex. Such projections are driven by the information content of each outcome with respect to the downstream task, which can be naturally measured via the Bayes loss of an optimal decision maker \citep{degroot1962uncertainty,zhao2021comparing}. Each fold hides information from downstream agents, gradually coarsening the support of context random variables, and allowing upstream predictive models to learn over sets with less data. The method notably decouples upstream prediction with downstream decision-making, allowing inspection of the learned abstractions used to drive policies.

The structure of the paper and key contributions are as follows. \S\ref{sec:decision_theoretic} contains background on decision-theoretic information, and describes the operational primitives of the novel class of {\tt ORIGAMI} algorithms. We also discuss the choice of projection operators and extensions of decision losses to sets. In \S\ref{sec:origami}, we detail three different objective functions to drive the projections, outlining computation-accuracy trade-offs. We further discuss a deep neural network surrogate for {\tt ORIGAMI} that can be trained to approximate the algorithm over a class of decision losses. In \S\ref{sec:experiments} we validate {\tt ORIGAMI} in data-limited deep active learning as in a closed-loop decision task involving wildfire management, where policies based on predictions over {\tt ORIGAMI} abstractions are shown to perform with lower losses.
\section{Background}\label{sec:problem_setting}
\paragraph{Notation}
Let $p(x,z)$ denote the underlying data generating process relating variables $x$ and outcome variables $z$, and $p(z|x)$ the conditional distribution over a finite set $\bZ$ of size $|\bZ| = C$. An agent observes $x$, and given a model $p_\theta(z|x)$ of $p(z|x)$, returns an action $a$ following the policy $\pi(Z)$. Domain-knowledge about the task is represented as a loss function $\ell:\bZ\times\bA\rightarrow\R$ measuring the cost of performing action $a$ when the outcome is $z$.

As both outcome and action spaces are assumed to be of finite dimensions, the loss function can be conveniently represented by a matrix $\bL\in\mathbb{R}^{|\bA|\times C}$ defined as
\[
    \bL_{ij} = \ell({z_i, a_j}).
\]
Further, $\bp = \{p(z_i|x)\}_i$ is a vector in $\R^C$ taking values in the probability simplex $\Delta^C$. In practice $p(x,z)$ is not known and we are given a dataset $\{x_k, z_k\}$ of samples, in addition to a decision loss $\ell$, with the final objective of identifying the best policy.

\paragraph{Problem setting}

We are interested in domains where the space of outcomes $\bZ$ for the random variable $Z$ is high-dimensional, e.g., the set of all possible medical conditions, or the space of geographical locations. As an example, consider the setting where a clinician is tasked with choosing an optimal treatment for a patient given the distribution $ \bp_{\theta}$ over a set of patient states, given measurements $x$. Here, optimizing a model for  $\bp_{\theta}$ or the policy $\pi$ can be challenging and require a large number of samples from $p(x,z)$. To overcome this limitation, we propose to reduce the dimensionality of $Z$ in a way that preserves as much useful information as possible for downstream decision-making. Our main targets are scalable methods, including approaches that generate explainable abstractions to enable compatibility with human decision-makers.

The core insight behind our approach is that not all the information contained in $\bp_{\theta}$ is necessary for decision-making. We take inspiration from human decision-making, where action under uncertainty appears to be taken swiftly with redundant information being abstracted away  \citep{lindig2019analyzing,ho2019value}.

\begin{figure}
    \centering
    \includegraphics[width=.5\linewidth]{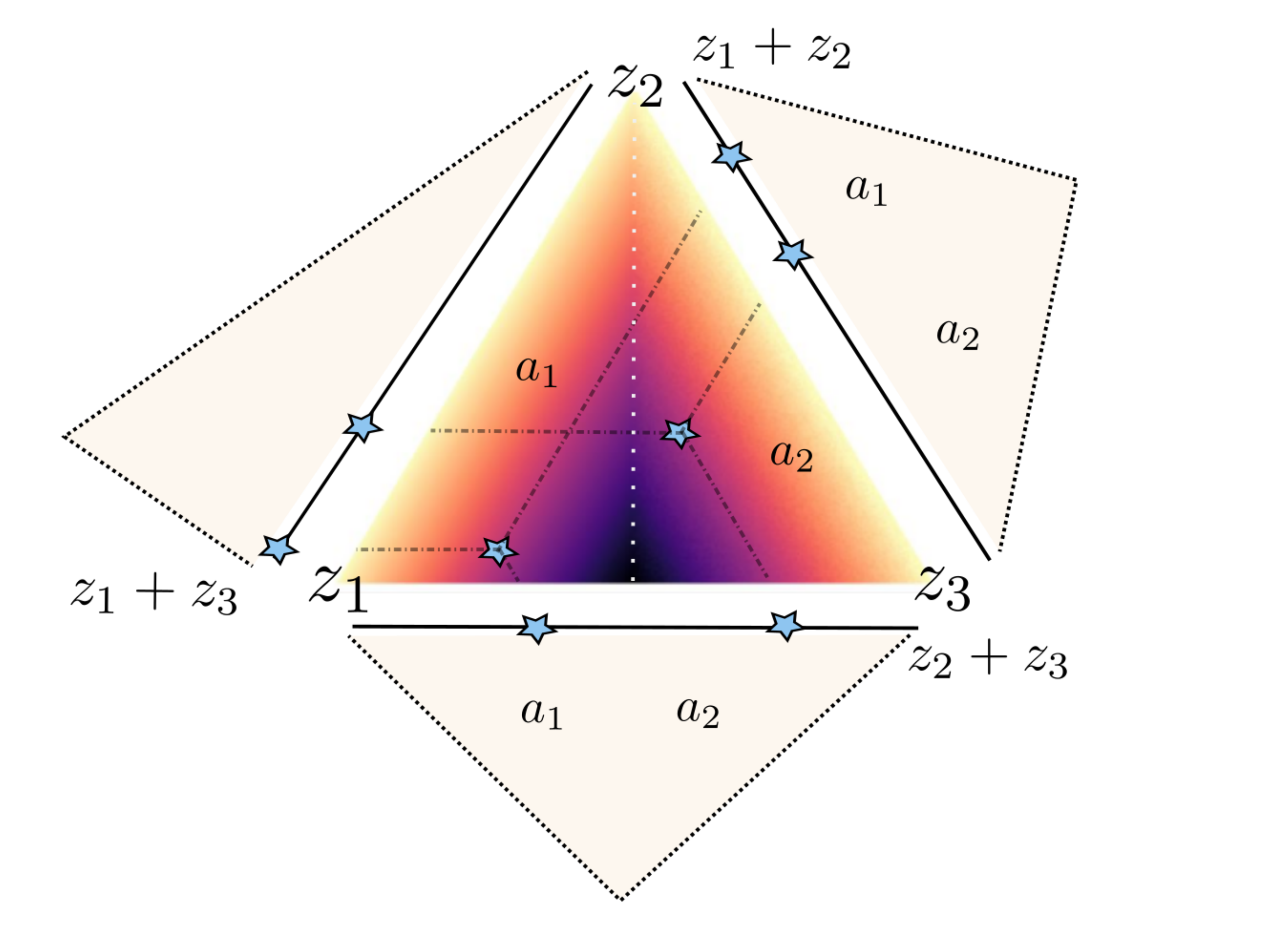}
    \vspace{-5.5mm}
    \caption{\footnotesize Folding the probability simplex can introduce a suboptimality gap in downstream decision-making. Some distributions $\bp$ remain on the same side of the decision boundary (in white), whereas others switch sides.}
    \label{fig:example_fold}
\end{figure}
\section{Decision-Theoretic Information}\label{sec:decision_theoretic}
Our goal to find a complete partition for the support of $p(z|x)$ i.e., $\cP(\bZ) = \{\bZ_k\}$, $\bZ_i \bigcap  \bZ_j = \emptyset$ for any $i \neq j$ and $\bigcup \bZ_k = \bZ$. Out of all possible partitions of the set, we seek those that minimally affect decision-making, as measured by the loss $\ell$. In other words, we aim to hide information that is not relevant to the decision task. A natural quantity to consider is the H-entropy \citep{degroot1962uncertainty,zhao2021comparing} of $p(z|x)$:
\begin{equation}\label{eq:hent}
    \begin{aligned}
    H_{\ell}(\bp) &= \inf_{a\in\cA}\bE_{p(z|x)}[\ell(Z, a)] \\
                &= \min_a (\bL \bp).
     \end{aligned}
\end{equation}
where $\bL \bp \in \R^{|\bA|}$. H-entropy is the Bayes optimal loss for an agent required to select an optimal action $a$ in expectation over $p(z|x)$, and generalizes other notions of information.

For convenience of notation, we will henceforth denote vectors $p(z|x)$ with $\bp$. Defining a partition $\cP(\bZ)$ naturally induces a distribution $\bq$ with support $\cP(\bZ)$. Thus, we can quantify the increase of H-entropy caused by partitioning the support $\bZ$:
\[
    \delta(\bq, \bp) = H_{\tilde \ell}(\bq) - H_{\ell}(\bp)
\]
which we refer to as the H-entropy suboptimality gap of $\cP(\bZ)$. For the above to be well-defined, we require a decision loss over the sets in $\cP(\bZ)$, denoted as $\tilde \ell$. We detail how to define set extensions of $\ell$ in Sec. \ref{ext:sets}.
\subsection{How to Fold a Simplex}

We cast the search for a partition $\cP(\bZ)$ through a geometric lens, leveraging the structure of the simplex $\Delta^C$. Our basic operation will involve \textit{folding} the simplex:
\begin{definition}[Simplex fold]\label{simplex_fold}
A fold is a map $f_{i\rightarrow j}: \Delta^C\rightarrow\Delta^{C-1};~\bp\mapsto \bq$ defined as 
\[
f_{i\rightarrow j} = 
\begin{cases}\bq_k = \bp_k \quad & \forall k \neq i, \forall k \neq j \\
\bq_j = \bp_i + \bp_j \quad & \text{otherwise}
\end{cases}
\]
\end{definition}

A fold projects elements of $\Delta^C$ onto $\Delta^{C-1}$. There is an intuitive interpretation for the output of a folding operation: two outcomes $z_i, z_j$ are grouped together into a set, and $\bq_j$ is the probability that either $z_i$ or $z_j$ occur. 

\begin{tcolorbox}[
    colback=blue!5,     
    sharp corners,
    boxrule=0.2mm]
\textbf{\textsf{Example:}} Consider a three-dimensional simplex $\Delta^3$, i.e., with $|\bZ| = 3$, and a loss function
\[
\bL = \begin{bmatrix}1&0&0\\0&0&1\end{bmatrix}
\]
The simplex and decision boundaries are visualized in Fig. \ref{fig:example_fold}. Along each side of the simplex, we show the decision loss over three projections to $\Delta^2$ obtained by summing
the probabilities of outcomes $z_1, z_2, z_3$ pairwise. Some points $\bp\in\Delta^C$ do not cross the decision boundary after the projection, whereas others do, introducing a suboptimality gap. The decision boundary is linear in this example since $|\bA| = 2$\footnote{The general case studied in this paper is $|\bA| > 2$, where the decision boundaries are piecewise-linear.}.
\end{tcolorbox}

\paragraph{Partition as a sequence of folds}
We uniquely identify a partition $\cP(\bZ)$ via the sequence of folds 
$$f_{i_N\rightarrow j_N} \circ \dots \circ f_{i_1\rightarrow j_1}$$
where $i_n,j_n$ are the folding indexes at algorithm iteration $n$. 
Consider the example in Figure \ref{fig:part_th_folds}, where the partition $\cP(\bZ) = \{{\tt 1}, \{{\tt 2},{\tt 3}\}, \{{\tt 4},{\tt 5}\}\}$ is identified through $f_{4\rightarrow 5}\circ f_{3\rightarrow 2}$.

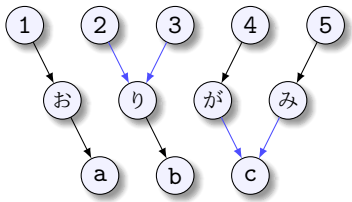
\begin{figure}[H]
    \centering
    \tikzset{vertex/.style={
    circle, 
    draw,
    fill=blue!5,
    inner sep=0pt, outer sep=0pt,
    blur shadow={shadow blur steps=5,shadow blur extra rounding=1.3pt},
    minimum size=0.5cm}
}
\begin{tikzpicture}
    \node [vertex] (v11) at (0, 0) {$\footnotesize\tt 1$};
    \node [vertex] (v12) at (1, 0) {$\footnotesize\tt 2$};
    \node [vertex] (v13) at (2, 0) {$\footnotesize\tt 3$};
    \node [vertex] (v14) at (3, 0) {$\footnotesize\tt 4$};
    \node [vertex] (v15) at (4, 0) {$\footnotesize\tt 5$};
    \def\dx{.5}
    \node [vertex] (v21) at (0 + \dx, -1) {\footnotesize お};
    \node [vertex] (v22) at (1 + \dx, -1) {\footnotesize り};
    \node [vertex] (v23) at (2 + \dx, -1) {\footnotesize が};
    \node [vertex] (v24) at (3 + \dx, -1) {\footnotesize み};
    \def\ddx{1}
    \node [vertex] (v31) at (0 + \ddx, -2) {$\tt a$};
    \node [vertex] (v32) at (1 + \ddx, -2) {$\tt b$};
    \node [vertex] (v33) at (2 + \ddx, -2) {$\tt c$};
    \draw[-latex] (v11)--(v21);
    \draw[-latex, blue!70] (v12)--(v22);
    \draw[-latex, blue!70] (v13)--(v22);
    \draw[-latex] (v14)--(v23);
    \draw[-latex] (v15)--(v24);
    \draw[-latex] (v21)--(v31);
    \draw[-latex] (v22)--(v32);
    \draw[-latex, blue!70] (v23)--(v33);
    \draw[-latex, blue!70] (v24)--(v33);
    %
    
    
    
    
    \end{tikzpicture}
    \vspace{-2mm}
    \caption{\footnotesize $\cP(\bZ) = \{{\tt 1}, \{{\tt 2},{\tt 3}\}, \{{\tt 4},{\tt 5}\}\}$ identified via two folds of $\Delta^5$. Name changes to classes indicate the implicit change of probability space after every fold.}
    \label{fig:part_th_folds}
\end{figure}

Iterative folding constructs a tree, starting from the $C$ vertices of $\Delta^C$ as leaves. Every fold adds a level, merging two nodes. At termination, each top-level node defines a set in the  final partition $\cP(\bZ)$, with elements identified as the leaves reachable from it.

\subsection{Computing Decision Losses on Sets}\label{ext:sets}

We seek an algorithmic procedure that iteratively folds the simplex until a reaching a stopping condition. We require an extension to $\ell$ that admits set-valued inputs. In the following, we consider the natural worst-case extension, 
\[
    \tilde \ell(S, a) = \max_{z\in S} \ell(z, a),\quad S\subset\bZ.
\]

The matrix representation follows by replacing, for each row $k$, column $i$ with the maximum of columns $i$ and $j$: $\tilde \bL_{kj} \leftarrow \max\{\bL_{ki}, \bL_{kj}\}$, in time $\Theta(|\bA|)$.

We note that, while theoretically possible, other free-form (mass preserving) projections may not have a sensible physical interpretation. Folding the simplex, as prescribed by \eqref{simplex_fold}, has the effect of grouping two outcomes together into a set, such that is remains possible to reason about worst-case decision losses. Numerically, this choice is key to preserving fast updates $\Theta(|\bA|)$ to the decision loss matrix $\bL$ required to compute $\ell$ over sets.
\paragraph{Properties of partitions} 
Assume to be given a perfect model of $\bp$ and a perfect decision making policy. Hiding information by partitioning its support can never improve the policy. Intuitively, this is due to the fact that information about events is now conveyed at a coarser level via sets $\{\bZ\}_k$ in the partition, rather than at the finer level of individual events. 

\begin{prop}[Folding increases H-entropy]\label{prop:increaseee}
Let $\bp\in\Delta^C$ and $\bq = f \circ f \dots \circ f(\bp)$ be any sequence of folds. Then,
\[
H_{\ell}(\bp) \leq H_{\tilde{\ell}}(\bq). 
\]
\end{prop}
In words, partitioning the support of $\bp$ raises the optimal lower bound decision loss. Specifically, the Bayes optimal loss lower bound increases due to the worst-case set extension. 
\begin{tcolorbox}[
    enhanced,
    breakable,
    colback=yellow!5,     
    sharp corners,
    boxrule=0.2mm]
\textbf{\textsf{Remark:}} Other set extensions for $\ell$ are available. We discuss a weighted sum extension where 
\[
    \tilde \ell(S, a) = \sum_{k: z_k\in S} \frac{\ell(z_k, a) \bp_k}{\sum_{k: z_k\in S}\bp_k}.
\]
In this case, prop. \ref{prop:increaseee} holds with equality: H-entropy is preserved by folding. How can the optimal decision loss not be affected by projecting the simplex down to a smaller dimension, effectively hiding information? This paradox is explained by noting that, through a fold and corresponding choice of set extension, one affects not only the \textbf{information content} but also \textbf{the downstream task}. For example, a summing extension 
\[
    \tilde \ell(S, a) = \sum_{z\in S} \ell(z, a).
\]
penalizes an outcome based on other outcomes in the same set, regardless of whether they occur or not. This results in penalizing larger sets in the partition $\cP(\bZ)$, biasing the projections found by the algorithm.

\end{tcolorbox}

Interestingly, we observe that the utility of decisions can improve if one optimizes a model $\bq_{\theta}$ on the \textit{lower resolution} support given by sets in $\cP(\bZ)$ rather than on the original support, particularly in data-limited regimes. An interpretation of this phenomenon is that partitioning into sets acts as a form of regularization for $\bq_{\theta}$ by hiding information  not relevant to the downstream task.

\section{{\tt ORIGAMI}: Algorithmic Folding}\label{sec:origami}
Each fold renders two outcomes indistinguishable from the perspective of the decision-maker. If $H_{\ell}(\bp) = H_{\tilde \ell}(f_{i\rightarrow j}(\bp))$, $z_i$ and $z_j$ are already equivalent for the decision task induced by $\ell$, and can thus be treated as a unique outcome without suboptimality. 
We have thus established a high-level desideratum for a folding algorithm: minimize at each step the suboptimality gap $\delta(\bp, \bq)$ induced by the projection. However, the gap discussed so far is local in the simplex, evaluating $\delta$ on two vectors $\bp \in \Delta^C$ and $\bq \in \Delta^{C-1}$. 

In practice, we have access to a dataset with samples from $p(z,x)$, yielding conditionals $p(z|x)\in\Delta^C$. Here, applying a fold $f_{i\rightarrow j}$ introduces a suboptimality gap at each point.

\paragraph{Folding objective}
Following this reasoning, one can cast each {\tt ORIGAMI} step as the following program:
\begin{equation}
    i^*, j^* = \arg\min_{i,j,i\neq j}\mathcal{L}(i, j, \bL)
\end{equation}
where $i^*, j^*$ are the indices of the optimal fold $f_{i^*\to j^*}$. The following discussion details three choices of objectives $\mathcal{L}$ that take into account different global information about the suboptimality induced by $f_{i\rightarrow j}$: \textit{total}, \textit{worst-case}, and \textit{vertex-only}.

\subsection{Integral Objective}\label{sec:integral}

The first objective relies on evaluating $\delta$ over the entire simplex:
\begin{equation}\label{eq:mc}
    \begin{aligned}
    \mathcal{L} = \frac{1}{\lambda}\int_{\Delta^{C}} \left[ H_{\tilde \ell}(f_{i \rightarrow j}(\bp)) - H_{\ell}(\bp) \right] \dd \bp. 
    \end{aligned}
\end{equation}
where $H_{\tilde\ell}$ denotes the H-entropy endowed with the \textit{folded} loss matrix, $H_{\tilde\ell}(f_{i\to j}(\bp)) = \min_{a}(\tilde \bL \bq)$.
This choice of objective corresponds to the $L_1$ norm of H-entropy increase and can be evaluated via Monte Carlo (MC) integration, thus requiring computationally costly sampling of $N$ vectors $\bp$ in $\Delta^C$ and evaluation of $\delta(f_{i\rightarrow j}(\bp), \bp)$ for all choices of $i,j$. By standard Law of Large Numbers arguments, the variance $\bV$ of a Monte Carlo estimate $\hat\mu_N$ of the total integral loss \eqref{eq:mc}
\[
    \hat{\mu}_N(i,j) = \frac{1}{N} \sum_{k=1}^N \left[ H_{\tilde \ell}(f_{i \rightarrow j}(\bp_k)) - H_{\ell}(\bp_k) \right]
\]
can be shown to converge linearly i.e., $\mathbb{V}[\hat\mu_N(i,j)] = \mathcal{O}({1}/{N})$ in the number of samples regardless of the dimension $C$.
\begin{prop}[Integral objective cost]
    {\tt ORIGAMI} driven by the objective \eqref{eq:mc}, with an $\epsilon$ requirement $\mathbb{V}[\hat{\mu}_N] \leq \epsilon$ has an asymptotic time cost of $\mathcal{O}(\frac{1}{\epsilon}|\bA|C^2)$.
\end{prop}
\begin{proof}
    We report here a proof sketch. For each pair of vertices in the simplex $\Delta^C$, $\frac{1}{2}C(C-1)$, we incur a cost $C^2$ to compute $\bL \bp$ and $|\bA|$ to find its minimum entry. This process has to be repeated ${1}/{\epsilon}$ for the variance of the MC estimate to be smaller than $\epsilon$.
\end{proof}
The minimization of the empirical estimate of \eqref{eq:mc} is then practically achieved by constructing the upper triangular portion of the matrix $\bm{M}_{ij} = \mu_N(i,j)$ and subsequently choosing the indices $(i^*,j^*)$ of the smallest entry of $\bm{M}$.

\clearpage

We report pseudocode below\footnote{The inner for-loop is fully parallelizable.}.

\begin{python}    
    # Fold with integral objective.
    # Input: $\Delta^c$, $L$, $N$.
    M = zeros(c, c) # 
    M = M + 10^4 # large initial distance
    p = uniform_sample(N, c) # on the simplex
    for (i, j) in combinations(range(c), 2):
        H_p = einsum("ac,bc->ba", L, p).min(dim=1) 
        q, Lt = fold(p, L, i, j) 
        H_q = einsum("ac,bc->ba", Lt, q).min(dim=1)
        M[i, j] = (H_q - H_p).mean(dim=0)
    i_fold, j_fold = argmin2d(M)
\end{python}

We further note that importance sampling and other variance reduction techniques may offer slight improvement to the convergence rate of $\hat{\mu}_N(i,j)$, reducing the overall cost of an {\tt ORIGAMI} fold. Instead, we leverage the structure of $H_{\ell}$ to develop alternative formulations to the integral objective.

\subsection{Max-Increase Objective}
Instead of the total loss of H-entropy (in a $L_1$ sense), we can choose folds that minimize the worst-case increase:
\begin{equation}
    \begin{aligned}
    \mathcal{L} &= \sup_{\bp\in\Delta^C} \left[H_{\tilde  \ell}(f_{i\rightarrow j}(\bp)) - H_{\ell}(\bp)\right] &~~ \Leftrightarrow\\
    &= \max_{\bp\in\Delta^C} \left[\min_{a}(\tilde{\bL} \bq)  - \min_{a}(\bL \bp)\right].
    \end{aligned}
\end{equation}
That is, the infinity norm of H-entropy increase induced by a fold $i \to j$. To find $i^*, j^* = \min_{i,j}\mathcal{L}(i,j,\bL)$ one has to solve, for each pair of indices, the inner optimization problem
\begin{equation}\label{eq:dc}
\max_{\bp\in\Delta^C} \left[\underbrace{\min_{a}(\tilde{\bL} \bq)}_{\text{concave}}  - \underbrace{\min_{a}(\bL \bp)}_{\text{concave}}\right]
\end{equation}
which belongs to the class of \textit{difference of convex or concave} (DC) problems \citep{hartman1959functions}. Here, we employ the \textit{concave-convex procedure} \citep{lipp2016variations}, a class of heuristic algorithms to find local solutions to DC problems.
\paragraph{Solving the inner-loop problem}

The simplest variant of a concave-convex procedure to compute $\cL$ starts by sampling an initial candidate maximizer $\bp^0\in\Delta^C$. Then, the candidate maximizer $\bp^{k}$ is updated as follows: the convex part of the problem is linearized around $\bp^k$,
\[  
    \hat H_\ell(\bp, \bp^{k}) =  \min_a \bL\bp^k + g^\top_k(\bp - \bp^k)
\]
where $g_k$ is a subgradient of $H_{\ell}$ i.e., $g_k\in \partial H_\ell(\bp^k)$. The candidate maximizer is then updated by solving the concave problem resulting from substituting $H_\ell(\bp)$ with its linearization, i.e.
\[
    \bp^{k+1} = \arg \max_{\bp\in\Delta^C} \left[\min_a \tilde\bL\bq - \hat H_\ell(\bp, \bp^{k}) \right]
\]
The algorithm is iterated until convergence, e.g., when the improvement in the true objective is less than a specified threshold. This adaptation of the convex-concave procedure to compute the objective for $\tt ORIGAMI$ folding leverages on the assumption that, at each step, the concavified problems can be solved efficiently (see \cite{lipp2016variations} for further details and variants of this method).

Similar to the integral loss case, the objective needs to be computed for each unordered tuple $(i,j)$ in order to chose the optimal folding.

\clearpage

\begin{python}
    # Fold with max-increase objective.
    # Input: $\Delta^c$, $L$.
    M = zeros(c, c) # 
    M = M + 10^4 # large initial distance
    for (i, j) in combinations(range(c), 2):
        p = solve_inner_dc_problem(c) 
        H_p = einsum("ac,bc->ba", L, p).min(dim=1) 
        q, Lt = fold(p, L, i, j) 
        H_q = einsum("ac,bc->ba", Lt, q).min(dim=1)
        M[i, j] = H_q - H_p
    i_fold, j_fold = argmin2d(M)
\end{python}
Note that if $|\bA| = 1$, $L \in \mathbb{R}^{1 \times C}$ is a vector and the inner problem is the linear program $\max_{\bp\in\Delta^C}[\tilde{L}q - L p]$.
\subsection{Vertex Objective}\label{sec:vertex}
Not all points on the simplex carry the same information for {\tt ORIGAMI}. Due to concavity, H-entropy is always minimized at a vertex of the simplex: 
\begin{prop}[H--entropy is minimized on vertices]
The minimizer $\bp^* = \arg\min_{\bp \in\Delta^C} H_{\ell}(\bp)$ is a vertex of $\Delta^C$.
\end{prop}
Therefore, we may wish to focus on the regions of the simplex corresponding to confident (peaked) predictions of the upstream model $p_\theta(z|x)$ i.e., close to the vertices. We propose an objective for {\tt ORIGAMI} where folding indices are obtained after comparing the H-entropy at all vertices:
\begin{equation}\label{vertex_loss}
  \mathcal{L} = |H_{\ell}(\bp^{(i)}) - H_{\ell}(\bp^{(j)})| = |\min_{a}(\bL_i) - \min_{a}(\bL_j)|  
\end{equation}
where $\bp^{(i)}$ and $\bp^{(j)}$ are in the vertex set of the simplex. The vertex loss can be computed efficiently in $\Theta(|\bA|C^2)$. In particular, it does not require updating $L$ for each pair $i, j$: the decision matrix is 
updated to $\tilde L$ only after optimal pair $i^*, j^*$ is found, in contrast to integral and max-increase objectives. 
\begin{python}
    # Fold with vertex objective.
    # Input: $\Delta^c$, $L$.
    M = zeros(c, c) # 
    M = M + 10^4 # large initial distance
    for (i, j) in combinations(range(c), 2):
        p, q = one_hot(i, c), one_hot(j, c)
        H_p = einsum("ac,bc->ba", L, p).min(dim=1) 
        H_q = einsum("ac,bc->ba", L, q).min(dim=1)
        M[i, j] = H_q - H_p
    i_fold, j_fold = argmin2d(M)
\end{python}

\paragraph{Setting a stopping condition}
${\tt ORIGAMI}$ iterations may be stopped after a predetermined number of folds, or alternatively after the total suboptimality gap $\delta$ reaches a tolerance threshold. Interestingly, 
other {\tt ORIGAMI} runs may also be recursively initialized within each set in the output partition of the first run, yielding a hierarchical tree-of-sets abstraction of $\bZ$.

\begin{figure*}[b]
    \centering
    \includegraphics[width=0.85\linewidth]{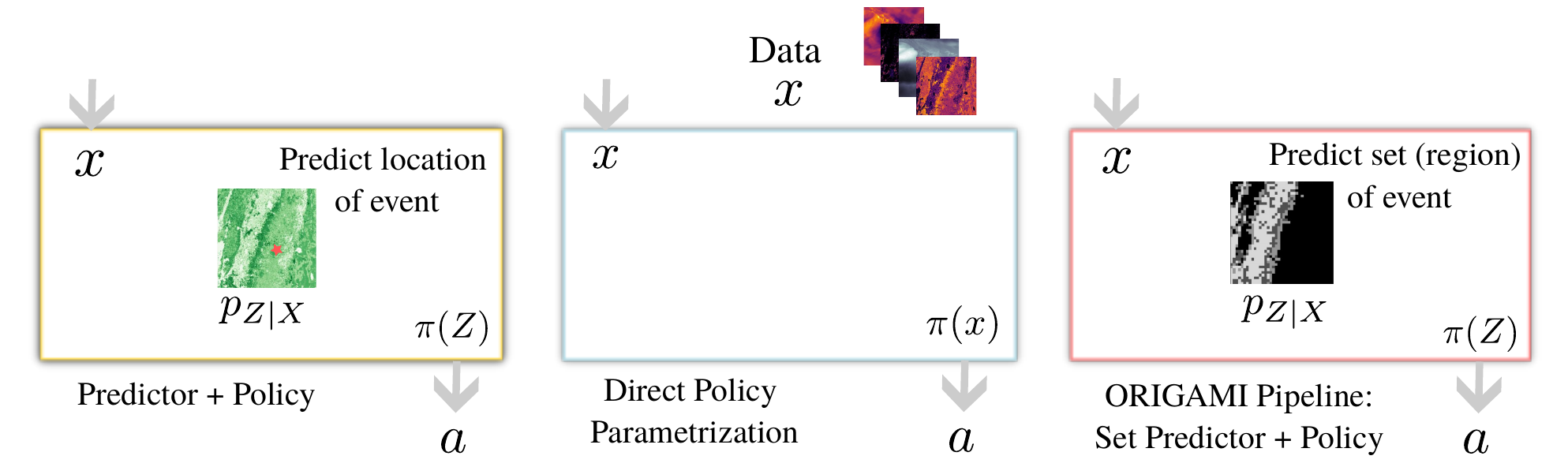}
    \vspace{-3mm}
    \caption{\footnotesize Applying {\tt ORIGAMI} to wildfire management, from prediction to action. \textbf{[Left]} {\tt Location predict} uses a location-level event predictor, then picks the best action that the predicted location \textbf{[Middle] } {\tt Direct policy} directly parameterizes the distribution over actions \textbf{[Right]} {\tt ORIGAMI} uses a region predictor, then picks the best action in the predicted region.}
    \label{fig:strategies}
\end{figure*}

\section{Numerical Experiments}\label{sec:experiments}

We now showcase how {\tt ORIGAMI} and set abstractions can be used in different learning contexts. The goal is to validate the scalability of {\tt ORIGAMI} to settings with thousands of outcomes, and to investigate whether abstractions improve downstream policies. If not specified, we use {\tt ORIGAMI} with the vertex objective.

\subsection{Folding for Decision Problems}

We evaluate support folding and {\tt ORIGAMI} in decision-making pipelines as a way to improve downstream policies. We consider wildfire management \citep{jain2020review}, and seek, in the frame of the definition of the problem, to identify a policy to minimize the damage caused by a wildfire at a given location.   

\paragraph{Experimental details}

We design and construct a new wildfire dataset named {\tt FIRE!} that contains information on active fires from \textit{Visible Infrared Imaging Radiometer Suite} (VIIRS), as well as climate \citep{hersbach2020era5}, vegetation, and topographic information \citep{rollins2009landfire}.  {\tt FIRE!} includes $1.3$ million fire instances collected over the years $2020$ and $2021$. We focus on a region in California. The overall dataset contains $29$ features, spanning climate and climate variables such as temperature and wind speeds, vegetation types and the radiative power of a given wildfire at each location. We aggregate temporal data in weekly periods, resulting with $102$ weeks between $2020$ and $2021$. In this case, the variable $Z$ indicates a geographical location, and we consider $1600$ possible locations (a discretized $40$ by $40$ grid). Additional details on the {\tt FIRE!} dataset are provided in the Appendix. 

\paragraph{Predictive task}
Each predictive model takes as input a snapshot $x$ ($1$ week, aggregated as described above) and is tasked with predicting whether the largest wildfire will occur in that particular location in the next week. Given a prediction, a policy picks among three wildfire management strategies: (1) sending a land team to actively suppress the fire, (2) sending aircraft, or (3) applying an indirect approach to slow down the spread \citep{national1996wildland}. For our example challenge problem, we craft a decision loss based on insights provided by \citep{national1996wildland}, where each strategy is weighted depending on various factors. For instance, sending a land team in regions with high altitudes and slopes might incur larger losses due to challenging terrain.\footnote{We note that our choice of decision loss serves as a proxy for expert decision losses and is not meant to be optimal or take into account every available factor.}

We formulate three different decision-making pipelines: {\tt Direct policy} parametrizes directly the distribution over actions, given $x$; {\tt Location predict} introduces a location predictor $p_\theta(z|x)$ trained on historical data, and a downstream policy $\pi(x) = \arg\min_{a}\mathbb{E}_{p_\theta(z|x)}\ell(Z, a)$.
{\tt ORIGAMI} is equivalent to  \textit{Location predict} except the model $q_\theta(\bZ|x)$ is trained on sets generated by folding geographical locations. The policy in this case involves computing the Bayes optimal action in each location $z \in \bZ$ of the predicted partition, then keeping the one most frequently optimal. Fig.\ref{fig:strategies} provides an overview of different approaches. Models $p_\theta$ and $q_\theta$ are parametrized as UNets \citep{zhou2019unet++}.

\paragraph{Results}
Summary results are provided in Table \ref{tab:result_wildfire}. We observe policies based on predictions over {\tt ORIGAMI} sets to achieve lower decision losses on our test data. Notably, \textit{Location predict} fails to correctly predict any wildfire location during testing, suggesting generalization at the fine-grained scale with the amount of data available is not possible. Predicting {\tt ORIGAMI} set membership ($5$ and $10$ sets) reaches a considerably higher accuracy. Fig.\ref{fig:setss} provides an example of the sets produced by {\tt ORIGAMI}: the regions (in shades of grey) are indicative of features the decision loss is based on.

\begin{figure}
    \centering
    \includegraphics[width=0.65\linewidth]{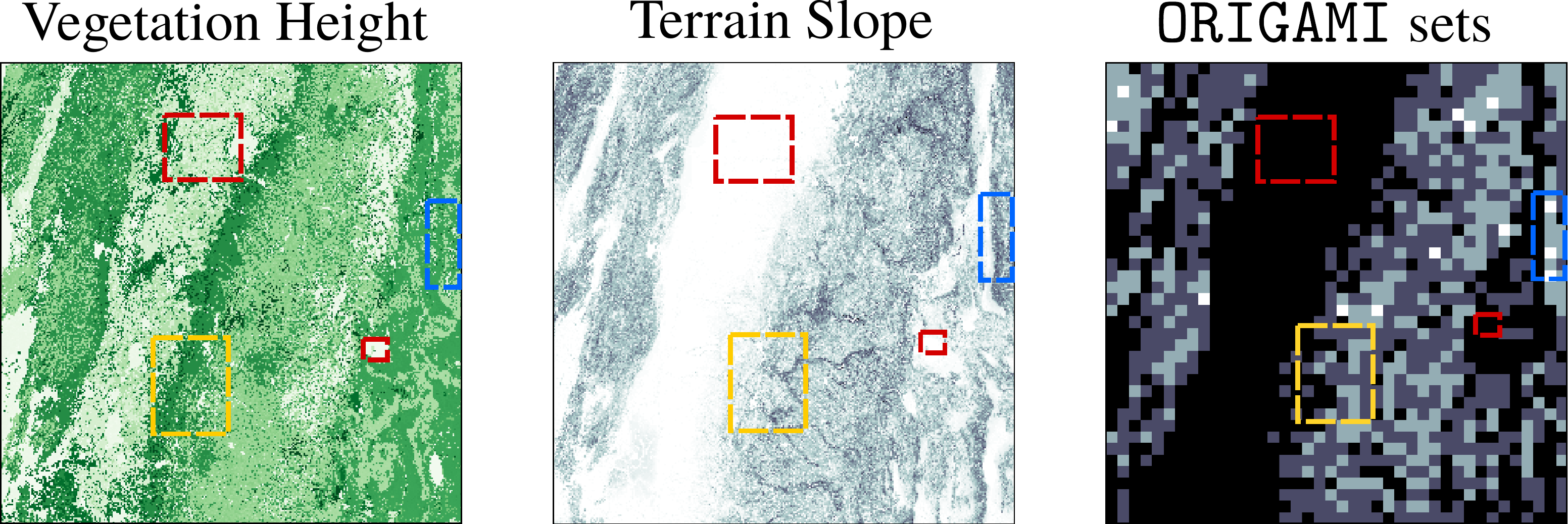}
    \vspace{-2.5mm}
    \caption{\footnotesize\textbf{[Left, Center]}: two features of the wildfire management dataset. \textbf{[Right]}: sets (regions) output of {\tt ORIGAMI} given a decision loss based on different features, including vegetation height and terrain slope. As is visible, the sets share common characteristics that are indicative of the features on the left e.g., red highlights regions of no slope and low vegetation, whereas blue highlights areas of tall vegetation.}
    \label{fig:setss}
\end{figure}

\begin{table}[t]
    \centering
    \begin{tabular}{c|cc}\toprule
        \textbf{Pipeline} & Predict acc. $\uparrow$ & Decision loss $\downarrow$ \\
        \midrule
        Random action & N/A & $0.820$ \\
        Direct policy & N/A & $0.731$ \\
        Location predict & $0$ & $0.723$\\
        \midrule
        \rowcolor{blue!3}
        {\tt ORIGAMI} (5)  & $\mathbf{67.4}$ & $0.707$ \\
        \rowcolor{blue!3}
        {\tt ORIGAMI} (10)  & $54.2$ & $\mathbf{0.701}$ \\
    \end{tabular}
    \vspace{-1mm}
    \caption{\small Benchmarking {\tt ORIGAMI} on wildfire management. Predicting sets of regions by quantizing space via {\tt ORIGAMI} folding induces higher quality downstream policy (as measured by the decision loss $\ell$). in {\tt ORIGAMI} (n), n refers to number of sets left at termination of the algorithm (we perform $C - n$ folds).}
    \label{tab:result_wildfire}
\end{table}

\subsection{Active Learning}

We apply {\tt ORIGAMI} to large neural network supervised training with limited data. In particular, we use the average H-entropy of sets generated by folding the simplex of classes as guidance to acquire additional data. The \textit{active learning} setting typically involves two interleaved stages: a training stage, where the network is optimized given available data, and an \textit{acquisition} stage, where a new batch is acquired\footnote{See  \citep{wang2016cost} for other acquisition strategies in deep active learning.}.

\paragraph{Experimental details}
In each run, we optimize an ensemble of $3$ {\tt ViT} \citep{dosovitskiy2020image} models for image classification on the standard {\tt CIFAR100} dataset. We start with a single batch of images, and each epoch we extend the dataset with an additional batch of $128$ images constructed following a particular procedure. We compare three different acquisition strategies: (1) random, in which we sample a new batch of images uniformly from all classes, (2) worst-$n$ class, which constructs a new sample of images from the $n$ classes with lowest marginal accuracy (3) {\tt ORIGAMI}, where we sample uniformly from the top set in the partition generated by {\tt ORIGAMI}, ranked by highest average H-entropy. To build {\tt ORIGAMI} sets, we use a decision loss where each model is an action, such that $|\bA| = 3$, $C = 100$, and each entry in $L$ is the average loss of each model on all instances of a given class.

\paragraph{Results}
We provide results in Table \ref{tab:active}. With $100$ epochs and training and a total dataset of $\approx 10$k images, we reach $35.7$\% accuracy when {\tt ORIGAMI} is used as the acquisition method. We observe a quick drop off when inspecting test performance on the worst classes ordered by marginal accuracy, with {\tt ORIGAMI} having an overall higher worst-case accuracy. Sampling according to highest H-entropy ensures marginal accuracy across classes is balanced, with new data acquired for classes on which the ensemble is struggling.

\begin{table}[h]
    \centering
    \begin{tabular}{c|ccc}\toprule
        \textbf{Acquisition} & All classes & bot-$50$ &  bot-$20$ \\
        \midrule
        Random & $17.6$\% & $6.5$\% & $2.8$\%\\
        Worst-$1$ & $30.3$\% & $12.4$\% & $6.1$\% \\
        Worst-$3$ & $30.4$\% & $14.3$\% & $7.0$\% \\
        \midrule
        \rowcolor{blue!4}
        {\tt ORIGAMI} & $\mathbf{35.7}$\% & $\mathbf{19.4}$\% & $\mathbf{10.6}$\% \\ 
    \end{tabular}
    \vspace{-1mm}
    \caption{\small Performance of different acquisition methods in the deep active learning experiments. We report average test accuracy across: all $100$ classes, worst $20$ classes, and worst $50$ classes. Worst-classes are identified by ordering based on marginal test accuracy.}
    \label{tab:active}
    \vspace{-2mm}
\end{table}

\subsection{Amortized {\tt ORIGAMI}}\label{sec:amortized}
The vertex objective introduced in \S\ref{sec:vertex} considerably improves the computation cost of obtaining good abstractions by means of iterative folding when compared to the other methods discussed. An alternative solution is to instead  amortize the cost of computing the Bayes optimal objective \eqref{eq:mc} by pre-training a neural network approximator to match it on a dataset of loss matrices. In the following, we discuss preliminary results and observations, emerging from training a simple neural network to fit the map $\bL \to \cL(\bL, i, j)$ on synthetic loss matrices $\bL$.

With $\bS(C)$ the space of all $C \x C$ upper triangular matrices ($\bS(C)\equiv \R^{C(C-1)/2}$), we define the neural network $a_\theta^C :\R^{|\bA|\x C}\to \bS(C)$ with parameters $\theta$.

We perform training by providing supervision to the model in the form of tuples $(\bL, \bm M_{ij})$, where the $\bm M_{ij}$ are produced \textit{offline} by the Monte Carlo approximation of the integral objective feeding $\bL\sim p(\bL)$ with $p(\bL) = \cU([0, 1]^{|\bA| \x C})$. 

The neural network parameters are then optimized via standard gradient methods to minimize a relative mean-square error objective between the model's predictions and target Bayes optimal folding costs. Such a model can be then invoked during iterative folding as a surrogate for other {$\tt ORIGAMI$} variants. 
Time and compute resources to build a dataset and train the model are thus traded for speedups at inference time when fast evaluation of the folding algorithm is prioritized. 
\paragraph{Experimental Details} 
We test the amortized procedure on a dataset of $10^4$ uniformly sampled loss matrices $\bL$. The number of actions $|\bA|$ is fixed to 2 while $C$ ranges from $3$ (the minimum significant number of classes) to $64$. This choice is due to the fact that all $\tt ORIGAMI$ algorithms scale linearly with the number of actions and we are mainly interested in amortizing the quadratic scaling with $C$. The ground-truth integral folding objectives in the form of the upper triangular matrices $\bm M_{ij}$ have then computed with the Monte Carlo procedure detailed in \S\ref{sec:integral} using $N = 10^3$ particles. The neural network $a_\theta^C$ comprises four layers with 64 neurons each. The loss matrices are flattened and passed to $a_\theta^C$ which returns vectors of dimension $C(C-1)/2$, corresponding to the predicted non-zero entries of $\bm M$.

\begin{figure}
    \centering
    \begin{tikzpicture}
\begin{axis}[
font=\scriptsize,
height=4.5cm, width=6.5cm,
legend cell align={left},
legend style={fill opacity=0, draw opacity=1, text opacity=1, draw=none},
xlabel={\sf Number of Classes $C$},
xmin=3, xmax=64,
xtick style={color=black},
ymin=-0.04915, ymax=1.03215,
ytick style={color=black}
]
\addplot [very thick, black]
table {%
3 0.0407479479908943
4 0.0447498857975006
5 0.0623383112251759
6 0.0811828523874283
7 0.103516526520252
8 0.118628762662411
16 0.386988997459412
32 0.491223067045212
64 0.50455915927887
};
\addlegendentry{\sf Test Loss}
\addplot [very thick, dotted]
table {%
3 0.983
4 0.928
5 0.839
6 0.741
7 0.64
8 0.556
16 0.036
32 0.005
64 0
};
\addlegendentry{\sf Test Accuracy}
\end{axis}

\end{tikzpicture}
    \begin{tikzpicture}
\begin{axis}[
font=\scriptsize,
height=4.5cm, width=6.5cm,
legend cell align={left},
legend columns=1,
legend style={
  fill opacity=0,
  draw opacity=1,
  text opacity=1,
  at={(0,1)},
  anchor=north west,
  draw=none
},
log basis y={10},
xlabel={\sf Number of Classes $C$},
xmin=3, xmax=8,
ylabel={\sf Fold Time $[s]$},
ylabel style={at={(-.22,.5)}},
ymin=1.46867900693809e-05, ymax=0.0198016525320284,
ymode=log,
]
\addplot [very thick, black]
table {%
3 0.000563902616500855
4 0.000838045597076416
5 0.00121562600135803
6 0.00169678354263306
7 0.0036574559211731
8 0.0142704923152924
};
\addlegendentry{\sf Integral}
\addplot [very thick, black, dotted]
table {%
3 0.000116840362548828
4 0.00011658763885498
5 0.000116155862808228
6 0.000118211269378662
7 0.000117031335830688
8 0.000117900848388672
};
\addlegendentry{\sf Vertex}
\addplot [very thick, blue!70]
table {%
3 2.20868587493896e-05
4 2.109694480896e-05
5 2.05049514770508e-05
6 2.03793048858643e-05
7 2.04916000366211e-05
8 2.05204486846924e-05
};
\addlegendentry{\sf Amortized}
\end{axis}

\end{tikzpicture}
    \vspace{-3mm}
    \caption{\footnotesize\textbf{[Left]}: Test RMSE loss between the output of $a^C_\theta(\bL)$ and $\bm M(\bL)$; accuracy in recovering $(i^*,j^*)$ using $a^C_\theta$ as surrogate objective for different numbers of classes. \textbf{[Right]}: Average (CPU) time required to obtain the optimal folding indexes $(i^*,j^*)$ for the integral (Monte Carlo), vertex and amortized $\tt ORIGAMI$.}
    \label{fig:time}
\end{figure}
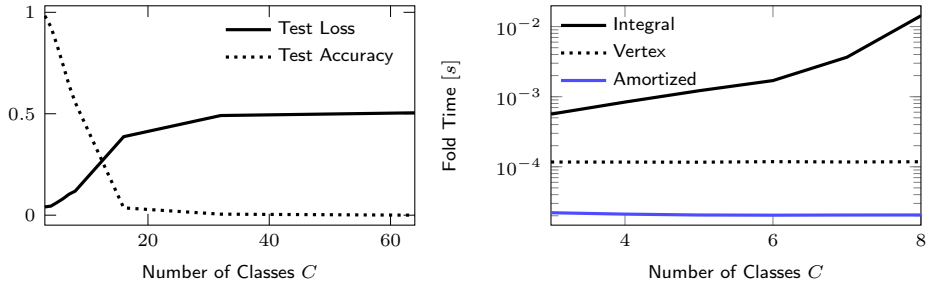

\paragraph{Results} The model, trained for 500 epochs for all values of $C$ is evaluated via a test set of $10^3$ additional tuples $(\bL, \bm M)$ in terms of RMSE loss and accuracy in predicting the optimal folding indices. We observe that the prediction accuracy rapidly decreases with the number of classes $C$ while the test RMSE loss increases, as reported by Fig.~\ref{fig:time}. This indicates that the learning problem becomes increasingly difficult with $C$ and the $\tt ORIGAMI$ folding cannot be amortized by a simple neural architecture. Nonetheless, in the region where the amortized model is accurate, i.e. for $C<8$, we report a significant speedup (several orders of magnitude) compared to $\tt ORIGAMI$ equipped with integral and vertex objectives.
\section{Related Work}

Multiple studies have taken a utility-theoretic perspective on learning and inference. The utility structure of problems has been leveraged in procedures for formulating abstractions of classes and actions as disjunctions \citep{horvitz1993abstract}. A decision-making perspective has also been used to guide abstraction for simplifying probabilistic inference \citep{poh1994}. Work includes efforts to drive the heterogeneous costs of misclassification into the objective functions and machine learning training procedures \citep{bach2006asymmutility}. 
Recent work has explored the end-to-end consideration of the quality of decisions in combinatorial optimization \citep{Wilder2019MeldingTD} and in human-AI collaboration \citep{wilder2020learning}. \cite{dubois2021lossy} propose to leverage knowledge of downstream tasks for compression, improving compression rates over task-agnostic methods. \cite{zhao2021comparing} formalize a new family of divergences, where discrepancy between distributions is measured through the optimal decision loss induced by each. H-entropy has seen use in Bayesian optimization \citep{neiswanger2022generalizing}, where a new family of acquisition functions is developed.

\section{Discussion}
We identify and outline several extensions related to the introduction of $\tt ORIGAMI$ algorithms in other deep learning domains, dynamic decision-making problems and numerical simulation.

\paragraph{Dynamic {\tt ORIGAMI}}
We have so far discussed static abstractions synthesized by $\tt ORIGAMI$ as a fixed set of sets of outcomes. However, as decision losses can change in time e.g., if decision matrix $L_t$ has an explicit dependence on time, the abstractions should track these new preferences. This can take place by applying a modified $\tt ORIGAMI$ algorithm able to unfold and fold, instead of starting anew each time.

\paragraph{Folding for simulation}
The process of quantization and creation of abstractions via {\tt ORIGAMI} can be loosely connected to meshing and discretization techniques ubiquitous in graphics and numerical simulation of differential equations \citep{plewa2005adaptive}. Instead of standard metrics to guide discretization, $\tt ORIGAMI$ is driven by utilities and is not constrained to sets that are local in space or time. As shown in our experiments, geographical regions found via $\tt ORIGAMI$ can involve disjoint subregions. Locality can be enforced or promoted via minimal changes to the method.

\paragraph{Adaptive tokenization}
The folding problem formalized by {\tt ORIGAMI} is closely related to tokenization procedures common in natural language processing and computer vision. It may be fruitful to investigate adaptive tokenization strategies using the machinery developed in this work, using other surrogates or estimators for the decision loss $L$.

\section{Conclusion}
We presented methods that guide the formulation of abstractions to simplify learning problems based on a careful consideration of downstream decisions.
The distillation of abstractions enables data-efficient learning of predictive models. We derive a class of iterative algorithms we refer to as {\tt ORIGAMI} that work to reduce the dimensionality of the probability simplex while preserving information useful for downstream decisions. In doing so, the method progressively hides information that is not necessary to implement optimal policies, allowing predictive models to learn over sets rather than fine-grained outcomes without loss in decision quality.

\bibliographystyle{abbrvnat}
\bibliography{main}
\clearpage
\appendix
\rule[0pt]{\columnwidth}{1pt}
\begin{center}
    \huge{Ideal Abstractions for Decision-Focused Learning} \\
    \vspace{0.3cm}
    \emph{Supplementary Material}
\end{center}
\rule[0pt]{\columnwidth}{1.5pt}
\doparttoc
\tableofcontents
\section{Derivations}

\begin{prop}[Folding increases H-entropy]\label{prop:increase}
Let $\bp\in\Delta^C$ and $\bq = f \circ f \dots \circ f(\bp)$ be any sequence of folds. Then,
\[
H_{\ell}(\bp) \leq H_{\tilde{\ell}}(\bq). 
\]
\end{prop}

\begin{proof}
We show the result for each row $k = 1, \dots, C$ of $\bL \bp$ and a single fold:
\[
\begin{aligned}
    \sum_{m=1}^{C} \bL_m \bp_m \leq \sum_{n=1}^{C-1} \tilde \bL_n\bq_n &\Leftrightarrow\\
    \bL_{ki} \bp_i + \bL_{kj} \bp_j \leq \max\{\bL_{ki}, \bL_{kj}\}\bq_j &\Leftrightarrow\\ 
    \bL_{ki} \bp_i + \bL_{kj} \bp_j \leq \max\{\bL_{ki}, \bL_{kj}\}\bp_i &~+\\ 
    \max\{\bL_{ki}, \bL_{kj}\}\bp_j. 
\end{aligned}
\]
\end{proof}

\begin{prop}[H--entropy is minimized on vertices]
The minimizer 
\[
    \bp^* = \arg\min_{\bp \in\Delta^C} H_{\ell}(\bp)
\]
is a vertex of $\Delta^C$
\end{prop}
\proof 
    Let $\bV_{\Delta}$ be the set of vertices of the simplex, i.e. the canonical basis $\eb_1,\dots, \eb_C$ of $\R^C$. 
    We need to show that $\min_{\bp\in\Delta^C} H_\ell(\bp) = \min_{i=1,\dots,C} H_\ell(\eb_i)$.
    Due to convexity of the simplex $\Delta^C$, the minimizer $\bp^*$ can be expresses as a convex combination of the vertices, i.e.
    \[
        \bp^* = \sum_{i=1}^{C}\alpha_i \eb_i,\quad\alpha_i\geq 0,~\sum_{i=1}^C\alpha_i = 1.
    \]
    By Jensen's inequality we have
    \[
        H_\ell(\bp^*) = \inf_a \bL \bp^* = \inf_a \bL \sum_{i=1}^{C}\alpha_i \eb_i \geq \sum_{i=1}^{C}\alpha_i \inf_a \bL\eb_i
    \]
    and 
    \[
        \sum_{i=1}^{C}\alpha_i \inf_a \bL\eb_i \geq \min_{i=1,\dots,C} \inf_a \bL \eb_i
    \]
    so the minimum of $H_\ell$ over $\bp\in\Delta^C$ is bounded below by the minimum
over the vertices. Since the vertices belong to $\Delta^C$, the result is proved.
\endproof

\section{Experiments}
\subsection{Folding for Decision Problems}
\paragraph{Dataset curation}

We design and build a new dataset named {\tt FIRE!} that contains active {\color{red!40!black}fire} information from \textit{Visible Infrared Imaging Radiometer Suite} (VIIRS) on a spatial resolution of $375$ meters, as well as {\color{blue!40!black}climate} and {\color{green!40!black}vegetation} data.

{\color{red!40!black}Fires: } We consider $1,339,234$ fire instances collected over the years $2020$ and $2021$. Each instance contains fire radiative power, location (latitude and longitude) and auxiliary information such as time of day and confidence for the measurement. We select the region spanned by latitude $(36, 39)$ and longitude $(-121.6, -118.6)$.

{\color{green!40!black}Vegetation: } We collect data from the LANDFIRE program. In particular, we add the following features: existing vegetation height (EVH), existing vegetation cover (EVC), existing vegetation type (EVT), slope degrees (SlpD), slope percent rise (slpP), roads, aspect (Asp). As these databases are updates at lower frequencies than VIIRs and climate, we have access to $2019$ and $2020$ snapshots which we use as additional context for the model. The region is aligned with VIIRS spatial coordinates.

{\color{blue!40!black}Climate: } We extract a set of climate and weather features from the large-scale ERA5 dataset. Weather and climate variables describe a larger region than latitude $(36, 39)$ and longitude $(-121.6, -118.6)$ to provide context for the predictive model. All data slices are aligned in time.

\paragraph{Predictive task}
The model takes as input a snapshot $x$ ($1$ week, aggregated as described above) and is tasked with predicting the location of the largest wildfire (measured in radiative power) in the following week. When using {\tt ORIGAMI}, spatial locations are clustered according to the decision loss, and thus the dimension of $y$ is smaller than the dimension of $x$. We optimize the parameters of all models using a standard binary cross entropy loss.

\paragraph{Decision making and decision loss}
We design a simple, deterministic closed-loop policy reliant on predictions made by a deep learning model. Our goal is to investigate whether the quality of a decision policy can be improved by performing upstream prediction on a "simplified" space of locations found as a decision--optimal clustering with {\tt ORIGAMI}.

Wildfire management actions are: (1) \textit{land intervention}, (2) \textit{aircraft intervention} (3) \textit{indirect containment}, according to the location predicted by the wildfire location model. The decision loss is crafted according to insights extracted from \citep{national1996wildland}. The following factors are used: fire radiative power, existing vegetation height, roads, temperature, magnitude of wind. In particular, action (1) incurs in high cost when slope and terrain height are larger, (2) when wind is strong, and (3) when vegetation is dense. Since all features are on different scales, we normalized the decision loss to obtain values in a comparable range. We note that the decision loss $\ell$ is not meant to encode \textit{all} factors one may want to consider for wildfire management, as the experiment is meant to showcase potential applications of {\tt ORIGAMI}.

\paragraph{Training details}

We train the wildfire location predictive model for $100$ epochs using the {\sc AdamW} optimizer~\citep{loshchilov2017decoupled}, with a learning rate of $0.001$ and a cosine decay schedule to $0.0001$. We set the batch size to $8$ (each element of the batch contains a temporal slice of context data, with the model asked to predict the location of the largest wildfire in the following week). All models are standard ${\tt ResNets}$ with $18$ layers.

When evaluating the policy, we pick the Bayes optimal wildfire strategy between: (1) \textit{land intervention}, (2) \textit{aircraft intervention} (3) \textit{indirect containment}, according to the location predicted by the wildfire location model.

\subsection{Active Learning}

\paragraph{Training details}

We train ensembles of $3$ vision transformers {\tt ViT} \citep{dosovitskiy2020image} on the CIFAR100 dataset, starting from a single random batch of data. Each epoch, we increase dataset size by sampling a new batch based on different acquisition methods:
\begin{itemize}[leftmargin=0.3in]
    \item \textit{random}: a new batch of images is obtained by sampling uniformly from all $100$ classes 
    \item worst-$n$: the new batch is obtained by sampling data uniformly if the corresponding label belongs to the $n$ classes with lowest marginal accuracy 
    \item {\tt ORIGAMI}: we apply {\tt ORIGAMI} to generate a partition of all classes. We select the set of classes in the partition with highest average H-entropy, and sample uniformly.
\end{itemize}
We use {\tt ViT}-base from the {\tt timm} library \citep{rw2019timm} with patch size $16$ and latent dimension $224$, and add the logits of each model in the ensemble before computing the cross-entropy loss. We train all models $100$ epochs using {\sc AdamW} optimizer~\citep{loshchilov2017decoupled}, a learning rate schedule with $20$ epochs of linear warmup ($0.0001$ to $0.001$) followed by cosine decay down to $0.0001$. We use batch size $128$.

\end{document}